\documentclass[letterpaper, 10 pt, conference]{ieeeconf}

\usepackage{amsmath}

\let\labelindent\relax
\usepackage[T1]{fontenc}
\usepackage{amsthm}
\usepackage{amsfonts}
\usepackage{amssymb}  
\usepackage{mathtools}
\usepackage[utf8]{inputenc}
\usepackage{comment}
\usepackage{algorithm,algpseudocode}
\algblock{Input}{EndInput}
\algnotext{EndInput}
\algblock{Output}{EndOutput}
\algnotext{EndOutput}

\usepackage{breqn}
\usepackage{graphicx}
\usepackage[version=4]{mhchem}
\usepackage{siunitx}
\usepackage{longtable,tabularx}
\usepackage{graphics} 
\usepackage{float}
\usepackage{epsfig} 
\usepackage{times} 
\usepackage{color}
\usepackage[dvipsnames]{xcolor}
\usepackage{pifont}
\usepackage{enumitem}
\usepackage{bm}
\usepackage{bbm}
\usepackage{soul}
\usepackage{balance}
\makeatletter
\newcommand{\multiline}[1]{%
  \begin{tabularx}{\dimexpr\linewidth-\ALG@thistlm}[t]{@{}X@{}}
    #1
  \end{tabularx}
}
\makeatother
\usepackage{hyperref}
\hypersetup{
    colorlinks=true,
    linkcolor=blue,
    filecolor=blue,      
    urlcolor=blue,
    pdftitle={Overleaf Example},
    pdfpagemode=FullScreen,
    }

\newcommand{\trace}{\textit{tr}}

\newcommand{\E}{\mathbb{E\,}}

\newcommand{\R}{\mathbb{R}}

\newtheorem{assumption}{Assumption}

\newtheorem{lemma}{Lemma}
\newtheorem{remark}{Remark}
\newtheorem*{theorem}{Theorem}

\newtheorem*{corollary}{Corollary}

\IEEEoverridecommandlockouts                              

\title{\LARGE \bf
Actively Learning Reinforcement Learning: A Stochastic Optimal Control Approach}

\author{Mohammad S. Ramadan$^1$,\,Mahmoud A. Hayajnh$^{2}$,\,Michael T. Tolley$^3$,\,Kyriakos G. Vamvoudakis$^2$
\thanks{
This work was supported in part, by Minerva under grant No. N$00014-18-1-2874$, by NSF under grant Nos. CAREER CPS-$1851588$,   CPS-$2227185$, S\&AS-$1849198$, and SATC-$2231651$, and by the Onassis Foundation-Scholarship ID: F ZQ $064-1/2020-2021$.}
\thanks{$^{1}$ Mathematics and Computer Science Division, Argonne National Laboratory, Lemont, IL 60439, USA, {\tt\small mramadan@anl.gov}.}
\thanks{$^{2}$ The Daniel Guggenheim School of Aerospace Engineering, Georgia Institute of Technology, GA 30332-0150 USA,  {\tt\small mhayajnh3@gatech.edu, kyriakos@gatech.edu}.}
\thanks{$^3$ Department of Mechanical and Aerospace Engineering, University of California, San Diego, CA 92161 USA,  {\tt\small tolley@ucsd.edu}.}}

\begin{document}
\maketitle
\thispagestyle{empty}
\pagestyle{empty}

\begin{abstract}
In this paper we propose a framework towards achieving two intertwined objectives: (i) equipping reinforcement learning with active exploration and deliberate information gathering, such that it regulates state and parameter uncertainties resulting from modeling mismatches and noisy sensory; and (ii) overcoming the computational intractability of stochastic optimal control. We approach both objectives by using reinforcement learning to compute the stochastic optimal control law. On one hand, we avoid the curse of dimensionality prohibiting the direct solution of the stochastic dynamic programming equation. On the other hand, the resulting stochastic optimal control reinforcement learning agent admits caution and probing, that is, optimal online exploration and exploitation. Unlike fixed exploration and exploitation balance, caution and probing are employed automatically by the controller in real-time, even after the learning process is terminated. We conclude the paper with a numerical simulation, illustrating how a Linear Quadratic Regulator with the certainty equivalence assumption may lead to poor performance and filter divergence, while our proposed approach is stabilizing, of an acceptable performance, and computationally convenient.
\end{abstract}

\section{Introduction}
The significance and proliferation of Reinforcement Learning (RL) across various disciplines are by now self-evident \cite{sutton2018reinforcement}. RL is an umbrella of algorithms that are rooted in the concept of stochastic approximation \cite{bertsekas2012dynamic,sutton2018reinforcement,recht2019tour}. At its core, it tries to solve an optimal control problem through maximizing some notion of a cumulative reward. Yet, with this promise of RL algorithms, a tough challenge arises when applying learned policies to real-world applications \cite{henderson2018deep}. These policies, trained in lab simulations or controlled environments, may in practice suffer a degradation in performance or exhibit an unsafe behavior \cite{recht2019tour}. This stems from modeling mismatches and discrepancies between the training environment and real-world conditions.

\begin{figure}[h]
    \centering
    \begin{footnotesize}
    \scalebox{.9}{
\begingroup%
  \makeatletter%
  \providecommand\color[2][]{%
    \errmessage{(Inkscape) Color is used for the text in Inkscape, but the package 'color.sty' is not loaded}%
    \renewcommand\color[2][]{}%
  }%
  \providecommand\transparent[1]{%
    \errmessage{(Inkscape) Transparency is used (non-zero) for the text in Inkscape, but the package 'transparent.sty' is not loaded}%
    \renewcommand\transparent[1]{}%
  }%
  \providecommand\rotatebox[2]{#2}%
  \newcommand*\fsize{\dimexpr\f@size pt\relax}%
  \newcommand*\lineheight[1]{\fontsize{\fsize}{#1\fsize}\selectfont}%
  \ifx\svgwidth\undefined%
    \setlength{\unitlength}{216bp}%
    \ifx\svgscale\undefined%
      \relax%
    \else%
      \setlength{\unitlength}{\unitlength * \real{\svgscale}}%
    \fi%
  \else%
    \setlength{\unitlength}{\svgwidth}%
  \fi%
  \global\let\svgwidth\undefined%
  \global\let\svgscale\undefined%
  \makeatother%
  \begin{picture}(1,0.83333333)%
    \lineheight{1}%
    \setlength\tabcolsep{0pt}%
    \put(0,0){\includegraphics[width=\unitlength]{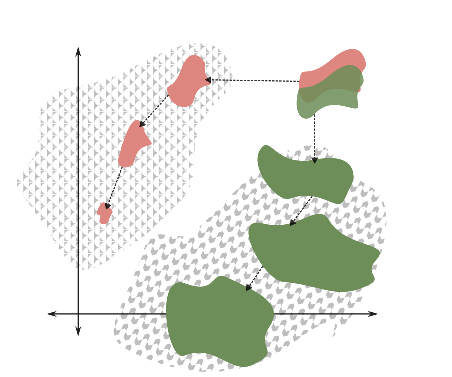}}%
    \put(0.06635696,0.77987718){\color[rgb]{0,0,0}\makebox(0,0)[lt]{\lineheight{1.25}\smash{\begin{tabular}[t]{l}high observability region\\(better/more sensors and/or higher SNR)\end{tabular}}}}%
    \put(0.66864614,0.06233861){\color[rgb]{0,0,0}\makebox(0,0)[lt]{\lineheight{1.25}\smash{\begin{tabular}[t]{l}low observability region\end{tabular}}}}%
    \put(0.03950029,0.05433946){\color[rgb]{0,0,0}\makebox(0,0)[lt]{\lineheight{1.25}\smash{\begin{tabular}[t]{l}state-space\end{tabular}}}}%
    \put(0.74868936,0.75181574){\color[rgb]{0,0,0}\makebox(0,0)[lt]{\lineheight{1.25}\smash{\begin{tabular}[t]{l}state uncertainty\end{tabular}}}}%
    \put(0.18307156,0.15470239){\color[rgb]{0,0,0}\makebox(0,0)[lt]{\lineheight{1.25}\smash{\begin{tabular}[t]{l}$o$\end{tabular}}}}%
  \end{picture}%
\endgroup%
}
    \end{footnotesize}
    \caption{\textbf{A graphical overview} of stochastic optimal control: the controller is not only concerned with regulating the state estimate (mean, mode, ... etc), but also regulating the state uncertainty (or the state estimate quality) via driving the system through high observable regions, for instance, regions of better signal-to-noise ratio (SNR) and/or of more/better sensors. The green and pink regions correspond to state uncertainty propagation along two different trajectories. The trajectory in pink, resembling a trajectory under stochastic optimal control, takes into consideration regulating uncertainty, and hence, in its path to the origin, it chooses the path of higher observability.} \label{fig:enter-label} \vskip-8mm
\end{figure}
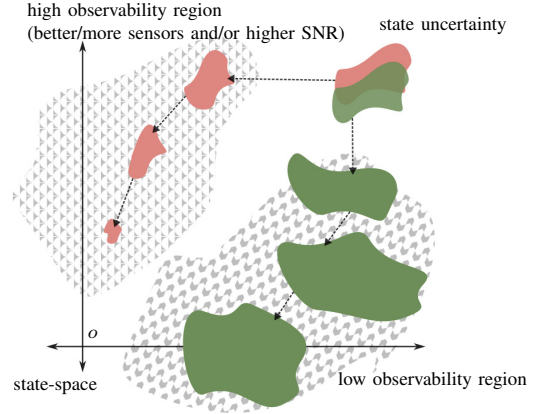

Early work in the RL community acknowledged the need for stochastic policies when the agent has limited or distorted access to the states \cite{jaakkola1994reinforcement}. The randomness introduced by stochastic policies diversifies chosen actions and hence achieves ``artificial exploration'' in a sense analogous to persistence of excitation in control theory and system identification \cite{marafioti2014persistently}. While this method diversifies actions to enhance learning, it may also compromise system safety and stability. Stochastic optimal control (SOC) \cite[ch.~25]{doucet2001sequential}, on the other hand, also known as dual control \cite{tse1973wide}, employs an automatic optimal exploration vs exploitation balance through employing two related key behaviors: caution and probing \cite{kumar2015stochastic}. Caution accounts for uncertainty when achieving safety and improving performance, while probing, often conflicting with caution, aims at gathering information about the system and reduce its uncertainties. These concepts play a central role in ensuring safety while enhancing the system's learning capabilities and state observability. However, in general, SOC is computationally prohibitive, except for the simplest cases \cite{aastrom1986dual}, due to the curse of dimensionality, which is further complicated by the the expectations over all sources of uncertainty.

The potential of both RL and SOC is limited by their inherent challenges discussed above. In this paper, we try to utilize each to mitigate the limitation of the other. Specifically, the modeling mismatch problems inherent to RL can be alleviated by the caution and probing effects of SOC. Caution imposes restrictions on the RL agent behavior under uncertainty and modelling mismatch, acting as a safeguard against false perception. Probing, on the other hand, aims to correct the modelling uncertainty and to create an accurate perception of the environment. In parallel to all of that, RL, possibly together with a neural net \cite{lillicrap2015continuous} as a function approximator, can mitigate the computational burden of SOC. These hypothesized mutual benefits serve as the motivation for the work we present here.

We target partially observed stochastic nonlinear systems with differentiable dynamics and quadratic cost functions. We show that the cost can be parameterized by the first two moments of the state, which we use the extended Kalman filter (EKF) to approximate. These EKF's approximate two moments as the state are then used as the new state used in the definition of the reinforcement learning agent, while the measurement plays the role of the external disturbance. We then employ the Deep Deterministic Policy Gradient algorithm \cite{lillicrap2015continuous} to solve the resulting reinforcement learning problem. We conclude the paper with a numerical simulation, illustrating how a Linear Quadratic Regulator with the certainty equivalence assumption may lead to poor performance and filter divergence, while our proposed approach is stabilizing, of an acceptable performance, and computationally convenient.

\section{Problem Formulation} \label{section: ProblemFormulation} 
Consider a class of discrete-time nonlinear systems described by, 
\begin{subequations}
\begin{align}
x_{k+1}&=f(x_k,u_k)+w_k,\label{NonSys_a}\\
y_k&=h(x_k)+v_k,\label{NonSys_b}
\end{align}\label{NonSys}
\end{subequations}
where $x_k\in\mathbb R^{r_x}$ is the state vector, $u_k\in\mathbb R^{r_u}$ is the control input, $y_k\in\mathbb R^{r_y}$ is the output signal, and $w_k\in\mathbb R^{r_x}$, $v_k\in\mathbb R^{r_y}$ are exogenous disturbances. The functions $f$ and $h$ are known and differentiable in $x_k$, and $f$ is twice differentiable in $u_k$. The disturbances $v_k$ and $w_k$ each is white, identically distributed with a continuous density of zero mean and positive definite covariances $\Sigma_w \succ 0$ and $\Sigma_v \succ 0$, respectively. These disturbances are independent from each other and from $x_0$, the initial state, which has a continuous density $\pi_{0\mid 0}$ of mean $x_{0 \mid 0}$ and covariance $\Sigma_{0 \mid 0} \succ 0$.

The goal is to construct a causal control law, i.e., a control law that is only dependent upon the data accessible up until the moment of evaluating the control action, or, $u_k = u_k (\mathcal Z _k)$, where $\mathcal{Z}_k=\{y_0,\hdots,y_k,u_0,\hdots,u_{k-1},\pi_{0 \mid 0}\}$. This law has to minimize the cost functional
\begin{align}\label{StochasticOptimalControlCost}
    J_N = \E \left \{ \gamma^N x_N^\top Q x_N \sum_{k=0}^{N-1} \gamma^k \left [x_k^\top Q x_k + u_k^\top R u_k\right ]\right\},
\end{align}
where $\gamma \in (0,1)$ is the discount factor, the inputs are restricted $u_k \in \mathbb U$, where $\mathbb U$ is a bounded set, and the weighting matrices $Q \succeq 0,\, R \succ 0$. The expectation is taken with respect to all the random variables, i.e., $x_0$, $w_k$, and $v_k$, for all $k$.

\begin{assumption} \label{Assumption_BIBO}
(Bounded input bounded state condition): The disturbance $w_k \subset \mathbb W \subset \R^{r_x}$ and the control input $u_k \in \mathbb U\subset \R^{r_u}$ belonging to the bounded sets $\mathbb W,\,\mathbb U$, result in an invariant compact set $\mathbb X$, i.e., $x_0 \in \mathbb X \implies x_k \in \mathbb X,\,\forall k \geq 0$.
\end{assumption}

\begin{remark}
In general nonlinear systems, as in \eqref{NonSys_a}, it is immaterial to separate states from unknown or time-varying parameters. Using the concept of the state augmentation \cite[p.~281]{jazwinski2007stochastic} the parameters of a system can be augmented in the state vector, which, in contrast to linear systems, augmentation for nonlinear systems does not fundamentally alter the structure of the system; it is nonlinear both ways. Therefore, we make no distinction between states and parameters in the following sections, hence, regulating state uncertainty encompasses parameter learning and modelling mismatch reduction \cite{aastrom1971problems}.
\end{remark}

\section{Background: stochastic optimal control}
The vector $x_k$ in \eqref{NonSys_a} retains its Markovian property due to the whiteness of $w_k$. Moreover, the observation $y_k$ is conditionally independent when conditioned on $x_k$; $v_k$ in \eqref{NonSys} is also white. These assumptions are typical in partially observable Markov decision processes \cite{bertsekas2012dynamic}. Under these conditions, the state $x_k$ cannot be directly accessed; it can only be inferred through the observation $y_k$, which is typically a lower order and/or distorted version of $x_k$. The vector $x_k$ is only a state in the Markovian sense, that is $$ p(x_{k+1}\mid x_k,x_{k-1},\hdots,x_0,u_k,\hdots,u_0) 
    = p(x_{k+1}\mid x_k,u_k).$$
However, for a decision maker or a control designer (or the learner as in \cite{jaakkola1994reinforcement}), an alternative ``state'' is required. That is, from a practical standpoint, a ``state'' in the sense of the minimal accessible piece of information adequate to reason about the future state trajectories. A possible ``state'' in this sense is the \textit{information state}, which is the state filtered density function $\pi_{k \mid k}=p(x_k \mid \mathcal{Z}_k)$ \cite{kumar2015stochastic}. As an ``informative statistic'' \cite{striebel1965sufficient}, it is sufficiently informative to enable the prediction of possible density propagation trajectories, and hence, can be used in the construction of the cost function. However, the information state, as a density function, is infinite dimensional in general, which renders its applicability infeasible, computationally. 

\subsection{Separation}

Adopting the information state $\pi_{k \mid k}$, a causal controller has the form $u_k = u_k (\pi _{k \mid k})$. This formulation of the control law allows the interpretation of SOC as comprising two distinct steps \cite[ch.~25]{doucet2000sequential}:
\begin{enumerate}
    \item Tracking $\pi_{k \mid k}$, that is, a Bayesian filter that propagates the information state \cite{kumar2015stochastic}.
    \item A law that assigns a value $u_k$ to each information state provided by the filter, such that this law minimizes \eqref{StochasticOptimalControlCost}.
\end{enumerate}

In the linear Gaussian state-space model case, the information state takes an equivalent finite dimensional characterization: the state conditional mean and covariance. If the system is unconstrained, the optimal control is indifferent to the state covariance and is only a function of the mean. This explains the separation principle in LQG control design \cite[ch.~8]{aastrom2012introduction}. This separation principle differs from that in the realm of SOC. The latter strictly denotes the two-step interpretation listed above.

As pointed out by \cite{tse1973wide}, tracking the information state $\pi_{k \mid k}$ is not half of the SOC problem; a convenient approximation to the Bayesian filter is typically less cumbersome than finding the SOC. The next subsection is a brief introduction to the Bayesian filter, which is then used to construct the dynamic programming equation for the stochastic case.

\subsubsection{The Bayesian Filter}
The equivalent stochastic representation of system \eqref{NonSys}: the transition kernel $x_{k+1} \sim p(x_{k+1} \mid x_k, u_k)$ and the measurement likelihood $y_k \sim p(y_k \mid x_k)$, can be achieved, similar to \cite{schon2011system}, due to the whiteness of $w_k$ and $v_k$.

The information state can be propagated through the Bayesian filter, which consists of the following two steps: the time update 
\begin{align}
p(x_{k+1}\mid u_k,\mathcal Z_{k}) &= \nonumber \\
& \hskip -10mm\int{p(x_{k+1}\mid u_k,x_k)p(x_k\mid \mathcal Z_{k})\,\textrm{d}x_k}, \label{T_update}
\end{align}
and the measurement update
\begin{align}
p(x_{k+1}\mid \mathcal Z_{k+1})&= \nonumber \\
& \hskip -15mm \frac{p(y_{k+1}\mid x_{k+1})p(x_{k+1}\mid u_k, \mathcal Z_{k})}{\int{p(y_{k+1}\mid  x_{k+1})p(x_{k+1}\mid u_k, \mathcal Z_{k})\,\textrm{d}x_{k+1}}}. \label{M_update}
\end{align}

Notice that to move from the filtered density at time-$k$ to $k+1$, the values of $u_k$ and $y_{k+1}$ are used. To simplify the notation, we denote $\pi_{k\mid k}=p(x_k\mid \mathcal Z_{k})$ and $\pi_{k+1\mid k}=p(x_{k+1}\mid u_k, \mathcal Z_{k})$, and define the mapping 
\begin{equation} \label{infoStateDynamics}
    \pi_{k+1\mid k+1}=T(\pi_{k\mid k},u_k,y_{k+1}),
\end{equation}
where $T$ maps $\pi_{k\mid k}$ to $\pi_{k+1\mid k}$ using $u_k$ in \eqref{T_update}, then to $\pi_{k+1\mid k+1}$ using $y_{k+1}$ in \eqref{M_update}.

\subsubsection{Stochastic Dynamic Programming}
A causal control law uses only the available information up to the moment of evaluating this law.  In accordance with the principle of optimality, when making the final control decision, denoted as $u_{N-1}$, and given the available information then $\mathcal{Z}_{N-1}=\{y_0,\hdots,y_{N-1},u_0,\hdots,u_{N-2},\pi_{0 \mid 0}\}$, the optimal cost (value function) can be determined as follows
\begin{align*}
V_{N-1}=\min_{u_{N-1} \in \mathbb{U}} \E \Big \{&x_{N-1}^\top Q x_{N-1} + u_{N-1}^\top R u_{N-1} + \\
&\hskip -7mm \gamma x_{N}^\top Q x_{N} \mid u_{N-1}, \mathcal Z_{N-1} \Big\},
\end{align*}
where the expectation is with respect to $x_{N-1}$ and $w_{N-1}$. Notice in the above expression, when $x_{N-1}$ is random, the disturbance $w_{N-1}$ is marginalized over through the expectation, and $u_{N-1}$ is the decision variable of the minimization, hence the expression is solely a function of the information state $\pi_{N-1 \mid N-1}$. That is,
\begin{align}
    &V_{N-1}(\pi_{N-1\mid N-1})= \nonumber\\
    &\min_{u_{N-1} \in \mathbb{U}} \E \Big \{x_{N-1}^\top Q x_{N-1} + u_{N-1}^\top R u_{N-1} + \nonumber \\
&\hskip 7mm \gamma x_{N}^\top Q x_{N} \mid u_{N-1}, \mathcal Z_{N-1} \Big\}. \label{eq:sDP_1}
\end{align}

Notice that if we define $V_N(\pi_{N\mid N})= \E \left \{ x_{N}^\top Q x_{N}\mid \mathcal Z_N\right\}$, it is not obvious how to put the stochastic dynamic programming equation above in a backward recursion form. For this purpose, we rely on the following result.

\begin{lemma} \label{lemma:smoothing}
(Smoothing theorem \cite[ch.~10]{resnick2019probability}): Let $(\Omega, \mathcal A, \mathbf{P})$ be a probability space, and $\chi:\Omega \to \R$ a measurable and $L^1$ function, i.e., $\E_{\mathbf{P}} | \chi | < \infty$, where $\E_{\mathbf{P}}$ is the expectation corresponding to $\mathbf{P}$. Let the $\sigma-$algebras $\mathcal{A}_0 \subset \mathcal{A}_1 \subset \mathcal{A}$, then
\begin{align*}
\E_{\mathbf{P}} \left \{ \chi \mid \mathcal{A}_1 \right \} = \E_{\mathbf{P}} \left \{ \E_{\mathbf{P}} \left \{ \chi\mid \mathcal{A}_0 \right \} \mid \mathcal{A}_1 \right \}. \qed
\end{align*}
\end{lemma}

\begin{corollary} \label{cor:smoothing_thm}
The expression
\begin{align*}
    \E \left \{ x_N^\top Q x_N \mid u_{N-1},\mathcal{Z}_{n-1} \right \} & \\
    &\hskip-25mm=\E \left \{\E \left \{ x_N^\top Q x_N \mid \mathcal{Z}_N \right \} \mid u_{N-1},\mathcal{Z}_{n-1}\right \},\\
    &\hskip-25mm=\E \left \{ V_N(\pi_{N \mid N})\mid u_{N-1},\mathcal{Z}_{n-1}\right \}. \qed
\end{align*}
\end{corollary}
Using the above result and \eqref{infoStateDynamics} in \eqref{eq:sDP_1}, we have 
\begin{align}
    &V_{N-1}(\pi_{N-1\mid N-1})= \nonumber \\
    &\hskip 7mm\min_{u_{N-1} \in \mathbb{U}} \E \Big \{x_{N-1}^\top Q x_{N-1} + u_{N-1}^\top R u_{N-1} + \nonumber \\
&\hskip 7mm \gamma V_N\left (T(\pi_{N-1\mid N-1},u_{N-1},y_N) \right) \mid u_{N-1}, \mathcal Z_{N-1} \Big\}. \label{eq:sDP_2}
\end{align}

This is the first backward iteration of the \textit{stochastic dynamic programming equation}, through which, a minimizing $u_{N-1}$ is assigned to each information state $\pi _ {N-1 \mid N-1}$. This step is repeated for all time-steps $N-2,\hdots,0$.

Solving the stochastic dynamic programming equation is computationally prohibitive, in general, primarily because of the infinite dimensionality of the information state. In the next section, we approximate the Bayesian filter by the EKF, reducing the information state into a finite-dimensional object.

\section{Methodology}
In this section we outline the EKF algorithm, and its ``wide sense'' (mean and covariance) approximation of the information state. We then adapt the cost in \eqref{StochasticOptimalControlCost} to the new approximate wide-sense information state. A few, mainly cosmetic, changes to this adapted cost are implemented to make it align with the assumptions/notation of the RL algorithm which will be outlined subsequently.

\subsection{EKF}
 We replace the infinite dimensional information state $\pi_{k \mid k}$ by a finite dimensional approximate one, namely, the state conditional mean vector $\hat x_{k \mid k}$ and covariance matrix $\Sigma_{k \mid k}$.

Let
\begin{align*}
    \hat x_{k \mid k} &= \E \left \{ x_k \mid \mathcal{Z}_k \right \}, \quad \hat x_{k \mid k-1} = \E \left \{ x_k \mid u_{k-1},\mathcal{Z}_{k-1}\right \},\\
    \Sigma_{k \mid k} &= \E \left \{ (x_k-\hat{x}_{k \mid k})(x_k-\hat{x}_{k \mid k})^\top \mid \mathcal{Z}_k \right \},\\
    \Sigma_{k \mid k-1} &= \E \left \{ (x_k-\hat{x}_{k \mid k-1})(x_k-\hat{x}_{k \mid k-1})^\top \mid u_{k-1},\mathcal{Z}_{k-1} \right \}.
\end{align*}
The EKF, similarly to the Bayesian filter, consists of the following two major steps:
\begin{itemize}
    \item Measurement-update
    \begin{align*}
        \hat x_{k\mid k}&=\hat x_{k \mid k-1} + L_k \left (y_k-h(\hat x_{k \mid k-1})\right),\\
        \Sigma_{k \mid k} &= \Sigma_{k \mid k-1} - L_k H_k \Sigma_{k \mid k-1}.
    \end{align*}
    \item Time-update
        \begin{align*}
        \hat x_{k+1\mid k}=f(\hat x_{k \mid k},u_k),\quad 
        \Sigma_{k+1 \mid k} = F_k\Sigma_{k \mid k}F_k^\top + \Sigma_w, 
    \end{align*}
\end{itemize}
which can be combined to write,
    \begin{align} \label{eq:EKF_stateDynamics}
        \hat x_{k+1\mid k+1}&=f(\hat x_{k \mid k},u_k) + L_{k+1} \left (y_{k+1}-g(\hat x_{k+1 \mid k})\right),\\
        \Sigma_{k+1 \mid k+1} &= \left (I - L_{k+1} H_{k+1} \right)\left (F_k\Sigma_{k \mid k}F_k^\top + \Sigma_w \right),
    \end{align}
where
\begin{align*}
 L_k&= \Sigma_{k \mid k-1} H_k^\top \left ( H_k \Sigma_{k \mid k-1} H_k^\top+\Sigma_v\right)^{-1}, \\
 F_k &= \left. \frac{\partial f(x,u_k)}{\partial x} \right | _{x=\hat x_{k\mid k}},\quad
 H_k = \left. \frac{\partial h(x)}{\partial x} \right | _{x=\hat x_{k\mid k-1}},
\end{align*}
and $\hat x_{0 \mid 0}$, $\Sigma_{0 \mid 0}$ are the initial state $x_0$ mean and covariance.

In general, the above conditional means and covariances are not exact; the state conditional densities $\pi_{k\mid k}$ are non-Gaussian due to the nonlinearity of $f$ and $g$. Hence, $\hat x_{k \mid k}$ and $\Sigma_{k \mid k}$ are merely approximations to the conditional mean and covariance of $x_k$ \cite{anderson2012optimal}. Define 
\begin{align}
    \hat \pi_{k+1} = \hat T ( \hat \pi_{k}, u_k, y_{k+1}), \label{eq:hatT}
\end{align}
where the tuple $\hat \pi_{k}=(\hat x_{k \mid k}, \Sigma_{k \mid k})$. Here $\hat T$ is a surrogate approximate mapping to $T$ in (\ref{infoStateDynamics}). The mapping $\hat T$ applies the steps \eqref{eq:EKF_stateDynamics} of the EKF to $\hat x_{k \mid k}, \Sigma_{k \mid k}$, using $u_k$ and $y_{k+1}$, and generates $\hat x_{k+1 \mid k+1}$ and $\Sigma_{k+1 \mid k+1}$.

\subsection{Cost}
The first two moments provided by the EKF are sufficient to evaluate the expectation of the finite-horizon cost function in \eqref{StochasticOptimalControlCost}, due to the quadratic stage costs.

We apply Lemma~\ref{lemma:smoothing} to the cost \eqref{StochasticOptimalControlCost}, conditioning on the information available at the time-step of each additive term. That is, we write each additive term as $\E \left \{ \E \{ x_k^\top Q x_k + u_k^\top R u_k \mid \mathcal{Z}_{k}\} \right \}$, per Lemma~\ref{lemma:smoothing}. If we then write $x_k = \hat x_{k \mid k} + \tilde x_k$, where $\E \{ \tilde x_k \mid \mathcal{Z}_{k}\} =0$, each additive term takes the form
\begin{align*}
    \E \left \{ \E \{ (\hat x_{k \mid k} + \tilde x_k)^\top Q (\hat x_{k \mid k} + \tilde x_k) + u_k^\top R u_k \mid \mathcal{Z}_{k}\} \right \}, \\ 
    =  \E \left \{  \hat x_{k \mid k}^\top Q \hat x_{k \mid k} + \trace(Q\Sigma_{k \mid k}) + u_k^\top R u_k \right \},
\end{align*}
after ignoring the zero mean cross-terms, and use the linearity of $\E$ and the circularity of the trace operator. Therefore we can re-write the cost \eqref{StochasticOptimalControlCost}
\begin{align}
J_N &= \E \Big \{ \gamma^N \left [ \hat x_{N\mid N}^\top Q x_{N\mid N} + \trace(Q \Sigma_{N \mid N}) \right ] + \nonumber \\
&\hskip -3mm\sum_{k=0}^{N-1} \gamma^k \left [\hat x_{k \mid k}^\top Q_k \hat x_{k \mid k} +   \trace(Q \Sigma_{k \mid k}) + u_k^\top R u_k\right ]\Big\},\label{eq:costWithTrace}
\end{align}
and the stochastic Dynamic Programming equation \eqref{eq:sDP_2} as
\begin{align}
    V_{N-1}(\pi_{N-1\mid N-1})&= \min_{u_{N-1} \in \mathbb{U}} \Big \{\hat x_{N-1\mid N-1}^\top Q \hat x_{N-1\mid N-1} + \nonumber\\
    &\hskip -20mm\trace(Q \Sigma_{N-1 \mid N-1}) + u_{N-1}^\top R u_{N-1} + \nonumber \\
&\hskip -20mm \gamma \E \left \{ V_N(T(\pi_{N-1\mid N-1},u_{N-1},y_N) \mid u_{N-1}, \mathcal Z_{N-1}\right \} \Big\}. \label{eq:sDP_3}
\end{align}
The following is a known result in SOC, which we present to show that minimizing \eqref{eq:costWithTrace} reduces to the LQG control in the linear case.
\begin{corollary}
(The separation principle \cite{aastrom2012introduction}): If the system \eqref{NonSys} is linear, the minimizing control law is $u_k = K_k \hat x_{k \mid k}$, where $K_k$ is the time-varying LQR gain of the deterministic (state fully observed and noise is zero) version of the problem and $\hat x_{k \mid k}$ is the conditional mean of the Kalman filter.
\end{corollary}
\begin{proof}
This is true since in the linear case $\Sigma_{k \mid k}$ evolves independently of $u_0,\hdots,u_k$, hence, can be omitted when minimizing the cost \eqref{eq:costWithTrace} for the control law. This results in the standard deterministic LQR problem.
\end{proof}

In the nonlinear case however, $\Sigma_{k \mid k}$'s evolution depends on the jacobians $F_k,\,H_k$, which in turn depend on $u_k$ and $\hat x_{k \mid k}$. Therefore, a SOC takes into consideration regulating the state uncertainty $\Sigma_{k \mid k}$ to achieve system and filter stability, and minimize the cost \eqref{eq:costWithTrace}.

\subsection{ Deterministic Policy Gradient}
Among the numerous algorithms within the scope of RL, we opt for the Deterministic Policy Gradient (DPG) algorithm \cite{silver2014deterministic}, in particular, its deep neural net implementation (DDPG) \cite{lillicrap2015continuous}. While we find this algorithm convenient within the context of this paper---primarily due to its ability to handle continuous state-action spaces---our choice does not impose strong preferences on the selection of other RL algorithms.

We first show that according to our assumptions and formulation, the stage-costs in the cost function \eqref{eq:costWithTrace} are bounded (by Assumption~\ref{Assumption_BIBO} and assuming uniform observability \cite{reif1999stochastic}), and hence the cost is bounded as $N \to \infty$ (sandwiched by a geometric series), i.e. $\lim_{N \to \infty} J_N \to J_\infty < \infty$. Furthermore, its minimum corresponds to the fixed point solution of \eqref{eq:sDP_3}, i.e., when $V_{N-1} = V_{N} = V_\infty$ \cite[ch.~7]{bertsekas2012dynamic}.

Having a well-defined cost and value functions over $N \to \infty$ is important as a huge portion of RL algorithms considers the infinite-horizon case. Another straightforward adaptation required is that the common notation in RL is to maximize the value, compared to minimizing the cost in optimal control. By simply defining the reward signal as the negative of the stage-cost in \eqref{eq:costWithTrace},
\begin{align*}
    r(\hat \pi_k,u_k) = - \E \left \{ \hat x_{k \mid k}^\top Q \hat x_{k \mid k} +   \trace(Q \Sigma_{k \mid k}) + u_k^\top R_k u_k \right \},
\end{align*}
which, as discussed above, is bounded for all $k$.

Analogous to \eqref{eq:sDP_3}, the state-action value function $\mathcal Q$ \cite{bertsekas2012dynamic},
\begin{align*}
    \mathcal Q(\hat \pi_k,u_k) =  r(\hat \pi_k, u_k) + \gamma\E \left(  \mathcal Q(\hat T(\hat \pi_k, u_k, y_{k+1}) \right),
\end{align*}
which its existence is immediate from the existence of $V_\infty$. We use a neural network control policy $u_k = \mu_\theta(\hat \pi_k)$, where $\theta$ denotes its weights and biases. Writing the reward representation of the cost \eqref{eq:costWithTrace} (its negative),
\begin{align*}
    J_\theta (\hat \pi_0) &=\E \Bigg \{\sum_{k=0}^\infty \gamma^k r(\hat \pi_k,\mu_\theta(\hat \pi_k)\Bigg\} = - J_\infty.
\end{align*}
The policy gradient theorems seek to find a description of the gradient $\nabla_\theta J_\theta$ which is convenient for computation.
We now present the Deterministic Policy Gradient Theorem of \cite{silver2014deterministic}, adapted for the case of SOC.
\begin{theorem}
(Deterministic Policy Gradient for SOC): Under the formulation in Section~\ref{section: ProblemFormulation}, Assumption~\ref{Assumption_BIBO}, filter stability and uniform observability (boundedness of \eqref{eq:EKF_stateDynamics}), the following identity holds true:
\begin{align*}
\nabla_{\theta} J_\theta = \mathbb{E}\Big[\nabla_{u} \mathcal Q(\hat \pi, u) \nabla_{\theta} \mu_\theta(\hat \pi)\Big].
\end{align*}
\end{theorem}
\begin{proof}
By the formulation in Section~\ref{section: ProblemFormulation} and the hypothesis above: (i) the transition density $ p(\hat \pi_{k+1} \mid \hat \pi_k,u_k)$, induced by the dynamics \eqref{eq:hatT}, is twice differentiable in $u_k$, and continuous in $\pi_{k+1}$ and $\pi_{k}$, (ii) the policy, being a neural net, is twice differentiable with respect to its parameters (for most activation functions), (iii) the reward function $r$, being quadratic, is differentiable in all of its arguments, (iv) the reward and the transition densities, being continuous, themselves and their jacobians in $u_k$, in $\hat \pi_k$ and $u_k$ over the compact set $\mathbb X \times \mathbb U$, are bounded. The points (i)-(iv) imply the regularity assumptions in \cite[Appendix~A]{silver2014deterministic}, and hence the result.
\end{proof}

In DDPG \cite{lillicrap2015continuous}, $\mathcal Q$ is also approximated by a neural net $\mathcal Q_{\psi}$ with parameters $\psi$ updated via temporal difference methods. While the above gradient is used to update the control policy neural net $\mu_\theta$. Algorithm~\ref{algorithm:ALRL} is the DDPG algorithm adapted for the information state (instead of the state). It does not include the target networks as in \cite{lillicrap2015continuous}, which can be augmented to Algorithm~\ref{algorithm:ALRL} to improve learning stability.

\begin{algorithm}[h]
\caption{Actively Learning RL via DDPG}\label{algorithm:ALRL}
\begin{algorithmic}[0]
\State Randomly initialize the weights $\theta,\,\psi$ of the neural nets $\mathcal Q_{\psi}$ and $\mu_\theta$;
\State Initialize replay buffer $\mathcal{R}$;
\For{$\text{episode}=1,2,\hdots$}
\State Randomly sample an initial information state $\hat \pi_0=\{\hat x_0, \Sigma_0\}$ and a true state $x_0$;
\For{$k=0,\hdots,N-1$}
\State Sample control actions $u_k=\mu_\theta(\hat \pi_k) + \eta$ where $\eta$ is an exploration noise;
\State Apply $u_k$ in \eqref{NonSys_a} to sample the true $x_{k+1}$;
\State Using $x_{k+1}$ in \eqref{NonSys_b}, sample the true $y_{k+1}$;
\State Using $\hat \pi_k$, $u_k$ and $y_{k+1}$, evaluate $\hat \pi_{k+1}$ using \eqref{eq:hatT};
\State Calculate the reward $r(\hat \pi_k, u_k)$;
\State Store the tuple $(\hat \pi_k, u_k, r(\hat \pi_k,u_k), \hat \pi_{k+1})$ in $\mathcal{R}$;
\State Sample a minibatch $\{(\hat \pi_i, u_i, r(\hat \pi_i,u_i), \hat \pi_{i+1}),\, i =1,\hdots,M\}$ of $\mathcal{R}$;
\State Set $z_i = r(\hat \pi_i,u_i) + \gamma \mathcal Q_\psi(\hat \pi_{i+1},\mu_\theta(\hat \pi_{i+1}))$;
\State Update the critic network $\mathcal Q_\psi$ by minimizing the loss $\frac{1}{M} \sum_{i=1}^M \left(z_i - \mathcal Q_\psi(\hat \pi_i, u_i)\right)^2$, w.r.t. $\psi$;
\State Update the policy network $\mu_\theta$ using the sample average policy gradient
$$ \nabla_\theta J_\theta \approx \frac{1}{M} \sum_{i=0}^M \nabla_u \mathcal Q_\psi(\hat \pi_i, u) \mid_{u=\mu_\theta(\hat{\pi}_i)}\nabla_\theta\mu_\theta(\hat \pi_i);$$
\EndFor
\EndFor
\end{algorithmic}
\end{algorithm}

\begin{remark}
The information state $\hat \pi_{k}$ contains repeated elements, since $\Sigma_{k \mid k}$ is symmetric \cite{anderson2012optimal}. We consider the upper triangle only in Algorithm~\ref{algorithm:ALRL}.
\end{remark}

\section{Numerical example} \label{Section: Numerical Examples}
In this section we implement Algorithm~\ref{algorithm:ALRL} on a system with varying state observability over the state-space $\R^3$. This simple example, although $3-$dimensional, but $9-$dimensional in the information state in \eqref{eq:hatT} ($3$ for the states and $6$ for the lower (or upper) triangle of the state covariance matrix), therefore, prohibitive for Dynamic Programming, and results in a complicated $(9+1)-$dimensional nonlinear program for Model Predictive Control. Instead, we demonstrate here how Algorithm~\ref{algorithm:ALRL} can be used to obtain a stochastic control. We compare the resulting closed-loop behavior to that obtained via LQG (LQG denotes an LQR with an EKF (in this example)).

Consider the model
\begin{align*}
    x_{k+1} &= 
    \begin{bmatrix}
    .92 & .2 & -.1\\
    0 & .95 & -.3\\
    0 & 0 & .93
    \end{bmatrix}
    x_k + 
    \begin{bmatrix}
    0 \\
    0 \\
    1
    \end{bmatrix}
    u_k + w_k,\\
    y_k &=\frac{1}{27} \text{ELU}((x_k{(1)})^3) + v_k,
\end{align*}
where $x_k(1)$ is the first entry of $x_k$, $w_k$ and $v_k$ obey the assumptions listed under \eqref{NonSys}, and moreover, $w_k \sim \mathcal{N}(0, \Sigma_w)$ and $w_k \sim \mathcal{N}(0, \Sigma_v)$\footnote{The notation $\mathcal{N}(\mu,\Sigma)$ denotes a Gaussian density with mean vector $\mu$ and covariance matrix $\Sigma$.}, where $\Sigma_v = 0.2$, $\Sigma_w = 0.5 \mathbb{I}_{3 \times 3}$. The ELU function is the exponential linear unit, an activation function heavily used in the deep learning context. It behaves as $x$ when $x>0$ and ELU$(x) \approx -1$ for $x<0$. More precisely, it is $e^x-1,\,x<0$ and $x,\,x\geq0$.

The choice of this system is due to its proliferation in system identification literature, as it belongs to the emerging Hammerstein–Wiener models, or linear dynamics composed with algebraic nonlinearities \cite{wills2013identification}. More importantly, this function admits variable $x_k$ to $y_k$ sensitivity, which in turn affects the observability of the system. For $x_k(1) < 0$, this sensitivity vanishes, and $x_k$ is no longer observables by the means of $y_k$. Therefore, it is the duty of a stochastic controller to take this into consideration: frequently seeking $x_k(1)>0$ for better observability (information gathering). Such a model can be effective in modeling deteriorating sensors quality or signal-to-noise ratio in a certain region of the state-space.

We use $Q=\mathbb I_{3 \times 3}$ and $R = 1$, and a discount factor $\gamma = 0.95$. The constraints are $u_k \in \mathbb U = [-5,5]$, which we enforce by using a saturated parameterized policy. 

Deterministic policy gradient methods \cite{silver2014deterministic} use the actor-critic learning architecture: the actor being the control policy, and the critic is its corresponding policy evaluation in the shape of an action-value function. In this example, both networks are feedforward, with one hidden layer of size $64$. The actor network receives nine inputs: the information state elements $\hat x _{k \mid k}$ and (the lower triangle of) $\Sigma_{k \mid k}\}$. The output of this network is the control $u_k$. The critic network, approximating the action value function, takes ten input entries: both of the information state elements as well as the corresponding control action $u_k$, and it outputs $\mathcal Q_\psi(\hat \pi _{k}, u_k)$. We use mini-batch learning, with batches of size $64$ of tuples $(\hat \pi_{k \mid k}, u_k, r_k, \hat T (\hat \pi_{k},u_k,y_{k+1}))$, and with learning rate $10^{-3}$. Figure~\ref{fig: Example1_reward} shows the statistics of the normalized accumulative reward of $50$ different training trials, all started from different randomized initial weights. The figure also shows the trial we picked to generate the subsequent closed-loop results. These $50$ training trials altogether consumed about $75$ minutes in computation, relying on an \textsc{NVIDIA V-100 GPU}\footnote{The results of this example can be reproduced using our open-source \textsc{Python/Pytorch} code: {\href{https://github.com/msramada/Active-Learning-Reinforcement-Learning}{https://github.com/msramada/Active-Learning-Reinforcement-Learning}}}.

\begin{figure}[h]
\centering 
\includegraphics{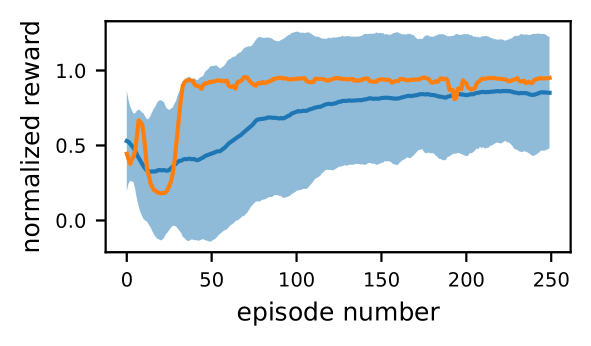} \caption{The average reward of a $50$ different runs of Algorithm~\ref{algorithm:ALRL} is shown in dark blue, while the shaded area is the corresponding two standard deviations about the average. In orange is the run with the highest terminal accumulative reward, which its corresponding controller is used to generate the closed-loop results below.\label{fig: Example1_reward}} \vskip-5mm
\end{figure}
\begin{figure}[h]
\centering 
\includegraphics{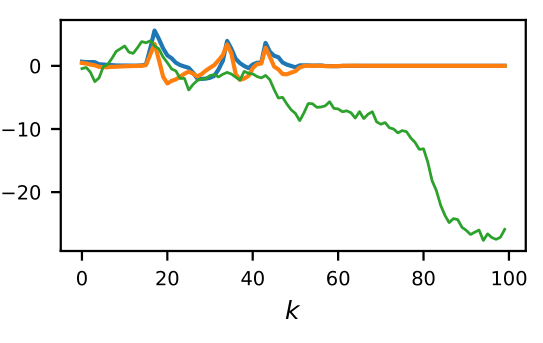}
\includegraphics{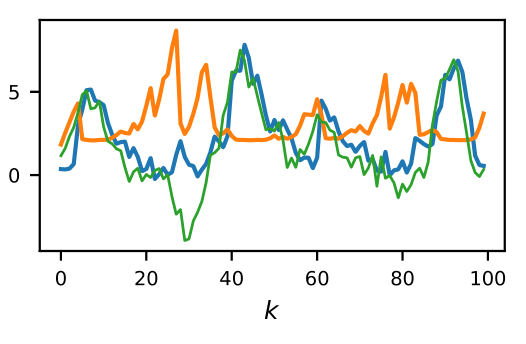} \caption{\textbf{LQG (upper) vs RL dual control (lower):} For each figure, the vertical axis is the magnitude of: the mean $\hat x_{k \mid k}(1)$ and $\trace(\Sigma_{k \mid k})$ which are shown in dark blue and orange, respectively, and the true state $x_k(1)$ shown in green.\label{fig: Example1_LQR_RL}}
\end{figure}

An LQG control is first applied: $u_k = K \hat x_{k \mid k}$, where $\hat x_{k \mid k}$ is provided by the EKF and $K$ is the LQR gain of the deterministic version of the system. The result of this LQG control is shown on the top of Figure~\ref{fig: Example1_LQR_RL}, which displays the true state deviation and \textit{filter divergence} \cite{anderson2012optimal,jazwinski2007stochastic}. This divergence is caused mainly by the LQG controller insisting on driving the state to the origin, making the system vulnerable to loss of observability if the true state $x_k(1)$ escapes to its negative side. This issue has been handled by the RL control resulting from our approach. In the bottom of Figure~\ref{fig: Example1_LQR_RL}, it can be seen that our controller does not prioritize only driving the state estimate to the origin, but also deliberately seeking some level of observability by continuously pushing $x_k(1)$ towards the positive side, whenever the uncertainty measure $\trace(\Sigma_{k \mid k})$ starts rising.

\section{Conclusion}
The presented framework is to produce an RL agent with attributes from SOC, namely, caution (to ensure safety and performance) and probing (to keep up active learning and information gathering). Our approach is built on the uniform observability and filter stability assumptions which are typically satisfied by the EKF. In general, however, what qualifies as a ``sufficient'' approximation to the information state is a rather complicated question, and one might be required to adapt our derivations to a different Bayesian filter (e.g. a Gaussian mixture), if the EKF is not sufficient.

In addition to the uncertainty propagator, our future work is aimed at crafting the reward signal to generate a different caution vs probing balance: for instance, prioritizing filter stability (or its accuracy) by including the true ($\lVert x_k - \hat x_{k \mid k} \rVert$) estimation error or by further penalizing the state covariance term in the reward signal.
\bibliographystyle{IEEEtranS}
\bibliography{References}

\end{document}